\documentclass{article}

 \usepackage[preprint]{neurips_2026}

\usepackage{amsmath,amssymb,amsthm,mathtools,amsfonts}
\usepackage[utf8]{inputenc}
\usepackage[T1]{fontenc}
\usepackage{lmodern}
\usepackage{microtype}
\usepackage{xcolor}
\usepackage{booktabs}
\usepackage{nicefrac}
\usepackage{url}
\usepackage{enumitem}
\usepackage{bbm}
\usepackage[colorlinks=true,linkcolor=blue,citecolor=blue,urlcolor=blue]{hyperref}
\usepackage{cleveref}
\usepackage{subcaption}
\usepackage{graphicx}
\usepackage{caption}
\usepackage{multirow}

\newtheorem{definition}{Definition}
\newtheorem{assumption}{Assumption}
\newtheorem{theorem}{Theorem}
\newtheorem{lemma}{Lemma}

\newcommand{\E}{\mathbb{E}}
\newcommand{\doop}[1]{\operatorname{do}(#1)}
\newcommand{\State}{\mathcal{S}}
\newcommand{\Act}{\mathcal{A}}
\newcommand{\SigmaSet}{\Sigma}
\newcommand{\GoalsN}{\Psi}
\newcommand{\GoalsAux}{\Psi_{\mathcal{F}}}
\newcommand{\Mem}{\mathcal{M}}

\title{What Must Generalist Agents Remember?}

\author{%
  Khurram Yamin\thanks{Equal contribution.} \\
  Carnegie Mellon University \\
  \texttt{kyamin@andrew.cmu.edu}
  \And
  Namrata Deka\footnotemark[1] \\
  Carnegie Mellon University \\
  \texttt{ndeka@andrew.cmu.edu}
  \And
  Maitreyi Swaroop\footnotemark[1] \\
  Carnegie Mellon University \\
  \texttt{mswaroop@andrew.cmu.edu}
  \AND
  Albert Ting \\
  Georgia Institute of Technology \\
  \texttt{ating7@gatech.edu}
  \And
  Jeff Schneider \\
  Carnegie Mellon University \\
  \texttt{jeff4@andrew.cmu.edu}
  \And
  Bryan Wilder \\
  Carnegie Mellon University \\
  \texttt{bwilder@andrew.cmu.edu}
}
\begin{document}

\maketitle

\begin{abstract}
This paper develops a formal account of what generalist agents must store in memory in order to act near-optimally across multiple environments and goals. It shows that when two domains share an observational bottleneck but require incompatible optimal actions, any uniformly near-optimal policy must induce distinct memory distributions at that bottleneck. The result yields a separation theorem: sufficiently successful agents cannot rely only on current state observations, but must preserve domain-relevant information in memory. The paper further shows that if an agent’s memory contains enough information to estimate values for related goals, then that memory can be used to approximately reconstruct the agent’s local transition dynamics. Together, these results characterize memory as the substrate that supports domain disambiguation, transition-model reconstruction, and planning for generalist agents.
\end{abstract}
\section{Introduction}\label{sec:intro}
Reinforcement learning is inherently causal: an agent's action is an intervention on the state trajectory, and the rewards it collects are downstream effects of that intervention. 
As we work towards general RL agents, we must contend with interventions not only \textit{by} the agent itself, but also \textit{on} its environment: the agent must pursue diverse goals within a single domain, and it must maintain competence when exogenous shifts alter the domain's dynamics. These two axes -- goal generality and domain robustness -- jointly define the setting we study.  
Under the assumption that the reward-producing mechanism is shared across domains, we ask:
\emph{what must such an agent's internal state encode about the world?}

Recent works in causal RL ask what an agent's \textit{policy} must identify in order to perform well across a family of tasks or environments.  
\cite{richens2025general} show that any agent satisfying a regret bound across a sufficiently rich family of multi-step goals in a fixed environment must have learned an approximate model of that environment's transitions. 
 \cite{richens2024robust} show that any agent satisfying a regret bound across a sufficiently rich family of interventions on the data generating process must have learned an approximate causal model, and \cite{ceriscioliagents} extend this to the mediated setting in sequential decision problems.
In each case, the agent is formalized as a mapping from a known task or domain index to a policy, and the result shows that this mapping must encode a world model that can be read off by querying the agent at different inputs.

We study a different setting, where the agent receives no domain index: it sees only its own trajectory, and must infer whatever it needs about the domain from online experience. The structure is not explicitly provided, but must be implicitly carried in the agent's memory.
This raises different questions from the abovementioned work, and ones we answer in this paper: \textit{what information must the agent's memory carry in order to act competently across goals and domains?} Further, \textit{when is that information enough to recover the domain's interventional dynamics?}
We address this question in a sequential, mediated setting. Our first result (\Cref{thm:separation}) is a memory-level necessary condition: whenever two domains require different actions at the same state, a uniformly near-optimal agent's memory must distinguish them at that state. This tells us that memory \textit{must} separate such domains, but not \textit{how}. Our second result (\Cref{thm:decoder}) provides the missing structure: if the goal family is rich enough for the memory to predict near-optimal action-values for a class of one-step predictive probes, then the memory admits a decoder of each domain's local interventional kernel. Both results constrain the same memory abstraction in the same multi-domain setting.

The paper is organized as follows: \Cref{sec:related} discusses the related work in causal RL and world model necessity. \Cref{sec:method} introduces the formal setup, assumptions, and notation, as well as our main results on representation separation (\Cref{thm:separation}) and decoder existence (\Cref{thm:decoder}). \Cref{sec:experiments} demonstrates the theorems' predictions in a gridworld. \Cref{sec:discussion} concludes by summarizing our results and discussing important questions for future work.






\section{Related Work}\label{sec:related}

\paragraph{World model necessity.} The most closely related line of work asks what a sufficiently competent agent must have learned about its environment. The classical statement is the good-regulator theorem \citep{conant1970every}: every good regulator of a system must be a model of that system. Recent work develop this for decision-making agents. \citet{richens2024robust} show that any agent achieving low regret across a family of interventions on the data generating process must have learned an approximate causal model of that process, with \citet{ceriscioliagents} extending this to sequential decision problems under mediation. In a parallel work, \citet{richens2025general} show that any agent satisfying regret bounds across a sufficiently rich family of multi-step goals in a fixed environment must have learned an approximate transition model of that environment. In each case the agent is formalized as a mapping from a known task or domain index to a policy, and the world model is read off by querying that mapping at different inputs. Our setting removes this explicit index: the agent observes only its own trajectory and so any world model information must be contained implicitly in memory rather than in the agent's input. This shifts the object of analysis from the task-to-policy mapping to the history/trajectory-to-memory map, and demands a memory-level necessity statement. 
\paragraph{Memory and latent context in RL.}
Our setting -- a family of MDPs that share structure but differ in dynamics through an unobserved factor that the agent must infer from experience -- is closely related to established formalisms in RL. It can be viewed as an instantiation of a partially observable MDP (POMDP) \citep{kaelbling1998planning}, and relates to formalisms for hidden-context inference such as Bayesian-adaptive MDPs \citep{duff2002optimal} and hidden-parameter MDPs \cite{doshi2016hidden} when the domain admits a parametric description.
A substantial empirical literature proposes methods for the resulting problem of inferring the latent context from trajectories. Recurrence-based methods \citep{hausknecht2015deep, duan2016rl2fastreinforcementlearning}, which we use as a baseline, carry context implicitly in an RNN's hidden state. Context-based methods \citep{rakelly2019efficient,zintgraf2021varibad} instead learn an explicit latent task variable with a dedicated inference objective. 
This literature proposes architectures and training objectives for obtaining agents that succeed when the relevant context is hidden.
Our contribution is complementary -- we do not propose an algorithm; instead we ask what \textit{any} agent that succeeds uniformly across this family must encode, and under what additional conditions that encoding admits a decoder of the local intervention dynamics.
\paragraph{Causal RL and distribution shift.}
More broadly, our work relates to causal approaches to reinforcement learning under distribution shift. Block MDPs \citep{du2019provably} and invariant prediction in block MDPs \citep{zhang2020invariant} use causal structure to identify state representations that generalize across environments; causal representation learning \citep{scholkopf2021toward} studies the recovery of latent causal variables from observation; and the transportability literature characterizes when causal effects identified in one domain transfer to another. These works are largely concerned with \textit{identification} given assumptions about structure; ours is concerned with what an agent's memory must encode in order to act competently when that structure is not given but must be inferred.
\section{Methodology}
\label{sec:method}

We study a goal-conditioned agent that must act near-optimally across both goals and domains. A domain, indexed by $\sigma\in\SigmaSet$, specifies an environment condition that may change the transition dynamics. The agent does not observe $\sigma$ directly, so any domain information needed for action selection must be inferred from its history. Equivalently, the setting is a latent-domain partially observable Markov decision process (POMDP). At the start of an episode, a domain $\sigma$ is selected but not revealed to the agent. Conditional on this domain, the process over observed states is a Markov decision process (MDP) with state $S_t$. Thus, the only hidden variable is the fixed domain index. If the raw observation is not Markov within a fixed domain, then $S_t$ denotes the augmented observed state for which the fixed-domain MDP description holds. The main question is what information the agent must store in memory to act well across domains and goals. We first prove a control-theoretic separation result: if two domains require different value-optimal actions at the same state for the same goal, then any uniformly near-optimal goal-conditioned policy must encode information that distinguishes those domains. This separation result does not identify the transition kernel. We then add auxiliary probe goals and a value-sufficiency condition under which the same memory representation supports approximate decoding of the local interventional dynamics.

\subsection{Setup}

We model interaction in discrete time with state space $\State$, action space $\Act$, and discount factor $\gamma\in[0,1)$. Let $\mathcal{P}(\State)$ denote the set of probability measures on $\State$. At time $t$, the agent observes $S_t\in\State$ and chooses $A_t\in\Act$. For each domain $\sigma\in\SigmaSet$, the initial state is drawn from $\mu_\sigma$, and transitions are governed by an interventional kernel $P^\sigma$. We write $\doop{A_t=a}$ to mean that action $a$ is externally set at time $t$. For any state $s\in\State$, action $a\in\Act$, and measurable set $B\subseteq\State$, define
\[
P^\sigma(B\mid s,a)
:=
\Pr_\sigma(S_{t+1}\in B\mid S_t=s,\doop{A_t=a}).
\]
Thus, $P^\sigma(\cdot\mid s,a)\in\mathcal{P}(\State)$ is the next-state distribution after action $a$ at state $s$ in domain $\sigma$. The kernel is not an observational distribution over actions; it describes the next-state law once an action has been fixed. Let $\GoalsN$ be the family of goals used in the separation result. Each goal $\psi\in\GoalsN$ is specified directly by a bounded per-transition utility $u_\psi:\State\times\State\to\mathbb{R}$. We assume there is $U_{\max}<\infty$ such that $|u_\psi(s,s')|\le U_{\max}$ for all $\psi\in\GoalsN$ and all $(s,s')\in\State\times\State$. These utilities are part of the goal specification and are not indexed by the domain; domain dependence enters through $\mu_\sigma$ and $P^\sigma$. Before choosing $A_t$, the agent observes history $h_t=(S_0,A_0,S_1,A_1,\ldots,A_{t-1},S_t)$. Given goal $\psi$, its policy is $\pi(\cdot\mid h_t;\psi)$. Hence, in domain $\sigma$, trajectories evolve as $S_0\sim\mu_\sigma$, $A_t\sim\pi(\cdot\mid h_t;\psi)$, and $S_{t+1}\sim P^\sigma(\cdot\mid S_t,A_t)$.
For policy $\pi$, domain $\sigma$, goal $\psi$, and initial state $s_0$, define
\[
V^{\pi,\psi}_{\sigma}(s_0)
:=
\E_{\pi,\sigma}\!\left[
\sum_{t=0}^{\infty}\gamma^t u_\psi(S_t,S_{t+1})
\,\middle|\,
S_0=s_0
\right],
\]
and
\[
V^{\star,\psi}_{\sigma}(s_0)
:=
\sup_{\pi}V^{\pi,\psi}_{\sigma}(s_0).
\]
For state-action values, define
\[
Q^\star_{\sigma,\psi}(s,a)
:=
\sup_{\pi}
\E_{\pi,\sigma}\!\left[
\sum_{k=0}^{\infty}\gamma^k u_\psi(S_k,S_{k+1})
\,\middle|\,
S_0=s,\doop{A_0=a}
\right].
\]
The intervention $\doop{A_0=a}$ fixes only the first action; later actions are chosen by $\pi$. When probabilities or expectations are conditioned on $s_0$, the process is started from $S_0=s_0$. The agent stores history in memory $M_t=f(h_t)\in\Mem$. The policy acts through a shared policy head $\pi_M$, so $\pi(a_t\mid h_t;\psi)=\pi_M(a_t\mid M_t,\psi)$. Therefore, two histories with the same memory state and goal induce the same action distribution.

\subsection{Control-theoretic notion of relevance}

We care only about domain differences that matter for action selection. Two domains may have different transition kernels but remain interchangeable if they induce the same near-optimal actions for every relevant goal. Fix $\varepsilon_{\mathrm{ctrl}}\ge 0$. For each domain $\sigma$, goal $\psi$, and state $s$, define the set of near-optimal actions:
\[
\mathcal{A}^{\varepsilon_{\mathrm{ctrl}}}_{\sigma,\psi}(s)
:=
\left\{
a\in\Act:
Q^\star_{\sigma,\psi}(s,a)
\ge
\sup_{a'\in\Act}Q^\star_{\sigma,\psi}(s,a')
-
\varepsilon_{\mathrm{ctrl}}
\right\}.
\]

\begin{definition}[Control equivalence]
\label{def:control_equiv}
Two domains $\sigma,\tilde{\sigma}\in\SigmaSet$ are \emph{control-equivalent} relative to $\GoalsN$ if $\mathcal{A}^{\varepsilon_{\mathrm{ctrl}}}_{\sigma,\psi}(s)=\mathcal{A}^{\varepsilon_{\mathrm{ctrl}}}_{\tilde{\sigma},\psi}(s)$ for every goal $\psi\in\GoalsN$ and every state $s\in\State$.
\end{definition}

Non-equivalence alone does not always force memory separation: if the near-optimal action sets overlap, one shared action may work in both domains. The separation result therefore uses domain pairs with disjoint value-gap-separated action sets at a shared bottleneck.

\begin{assumption}[Shared bottleneck with disjoint action sets]
\label{ass:shared_bottleneck}
Let $\mathcal{D}_{\mathrm{inc}}\subseteq\SigmaSet\times\SigmaSet$ denote a specified collection of domain pairs for which separation is required. For each pair $(\sigma,\tilde{\sigma})\in\mathcal{D}_{\mathrm{inc}}$, assume there exist a goal $\psi\in\GoalsN$, a state $s^\dagger\in\State$, times $t_\sigma,t_{\tilde{\sigma}}\in\{0,\ldots,T\}$, and nonempty measurable action sets $G_\sigma,G_{\tilde{\sigma}}\subseteq\Act$ satisfying the conditions below, where the constants $\rho>0$, $T\in\mathbb{N}$, $\bar\varepsilon_{\mathrm{opt}}\ge 0$, and $\varepsilon_{\mathrm{gap}}>0$ are uniform over all pairs in $\mathcal{D}_{\mathrm{inc}}$.

\begin{enumerate}[leftmargin=1.25em]
\item \textbf{Any sufficiently value-near-optimal policy reaches the bottleneck.}
For each $\xi\in\{\sigma,\tilde{\sigma}\}$, every policy $\pi$ satisfying
\[
V^{\star,\psi}_{\xi}(s_0)-V^{\pi,\psi}_{\xi}(s_0)
\le
\bar\varepsilon_{\mathrm{opt}}
\]
also satisfies $\Pr_{\xi}(S_{t_\xi}=s^\dagger\mid \pi,s_0)\ge \rho$.

\item \textbf{The action sets are disjoint and separated by a value gap.}
The sets $G_\sigma$ and $G_{\tilde{\sigma}}$ are disjoint. Moreover, for each $\xi\in\{\sigma,\tilde{\sigma}\}$, every action in $G_\xi$ is better by at least $\varepsilon_{\mathrm{gap}}$ than every action outside $G_\xi$ at $(s^\dagger,\psi)$:
\[
Q^\star_{\xi,\psi}(s^\dagger,a)
\ge
Q^\star_{\xi,\psi}(s^\dagger,b)
+
\varepsilon_{\mathrm{gap}}
\qquad
\forall a\in G_\xi,\ \forall b\notin G_\xi.
\]
\end{enumerate}
\end{assumption}

Thus, near-optimal behavior for each pair must pass through the same bottleneck state, but the value-gap-separated action sets at that state are disjoint across the two domains. The times $t_\sigma$ and $t_{\tilde{\sigma}}$ need not be equal; the shared bottleneck is the state $s^\dagger$, not necessarily a common clock time.

\subsection{Representation separation}

\begin{theorem}[Uniform cross-domain competence forces memory separation]
\label{thm:separation}
Assume the latent-domain model described above and Assumption~\ref{ass:shared_bottleneck}. Suppose a single goal-conditioned policy $\pi$ is uniformly value-near-optimal across all goals and domains: for some $\varepsilon_{\mathrm{opt}}\in[0,\bar\varepsilon_{\mathrm{opt}}]$,
\[
V^{\star,\psi'}_{\sigma}(s_0)-V^{\pi,\psi'}_{\sigma}(s_0)
\le
\varepsilon_{\mathrm{opt}}
\qquad
\forall \psi'\in\GoalsN,\ \forall \sigma\in\SigmaSet.
\]
Let $\varepsilon_0:=\tfrac12\rho\gamma^T\varepsilon_{\mathrm{gap}}$. If $\varepsilon_{\mathrm{opt}}<\varepsilon_0$, then for every specified pair $(\sigma,\tilde{\sigma})\in\mathcal{D}_{\mathrm{inc}}$ satisfying the shared-bottleneck conditions of Assumption~\ref{ass:shared_bottleneck}, the policy's memory must separate those domains at the corresponding bottleneck times.

More precisely, let $\psi$, $s^\dagger$, $t_\sigma$, $t_{\tilde{\sigma}}$, $G_\sigma$, and $G_{\tilde{\sigma}}$ be the witnesses from Assumption~\ref{ass:shared_bottleneck}. Then the conditional memory laws at the possibly different times $t_\sigma$ and $t_{\tilde{\sigma}}$ satisfy
\[
\Pr_\sigma(M_{t_\sigma}\in\cdot\mid S_{t_\sigma}=s^\dagger,\pi,s_0)
\neq
\Pr_{\tilde{\sigma}}(M_{t_{\tilde{\sigma}}}\in\cdot\mid S_{t_{\tilde{\sigma}}}=s^\dagger,\pi,s_0).
\]

In fact, writing
\[
\eta_\sigma:=\frac{\varepsilon_{\mathrm{opt}}}{\rho\gamma^{t_\sigma}\varepsilon_{\mathrm{gap}}},
\qquad
\eta_{\tilde{\sigma}}:=\frac{\varepsilon_{\mathrm{opt}}}{\rho\gamma^{t_{\tilde{\sigma}}}\varepsilon_{\mathrm{gap}}},
\]
we have
\[
\mathrm{TV}\!\left(
\Pr_\sigma(M_{t_\sigma}\in\cdot\mid S_{t_\sigma}=s^\dagger,\pi,s_0),
\Pr_{\tilde{\sigma}}(M_{t_{\tilde{\sigma}}}\in\cdot\mid S_{t_{\tilde{\sigma}}}=s^\dagger,\pi,s_0)
\right)
\ge
\left[
1-\eta_\sigma-\eta_{\tilde{\sigma}}
\right]_+,
\]
where $\mathrm{TV}$ denotes total variation distance and $[x]_+:=\max\{x,0\}$.
\end{theorem}

If two domains require disjoint value-gap-separated actions at the same bottleneck state, then any uniformly near-optimal agent with a shared policy head must distinguish those domains in memory. The comparison is not synchronous in clock time: it compares the conditional memory distribution in each domain after reaching the shared state $s^\dagger$ at that domain's bottleneck time, $t_\sigma$ or $t_{\tilde{\sigma}}$. This goes beyond the usual POMDP observation that a belief state can represent latent context: POMDP theory provides a sufficient statistic for optimal control, whereas Theorem~\ref{thm:separation} gives a necessity result for learned representations. The result applies to any history-to-memory map, without assuming Bayesian belief updates, a correct latent-state model, or an explicit context variable. Near-optimality itself forces the two conditional memory distributions to separate, with a total-variation lower bound determined by the bottleneck reach probability, discounting, value gap, and optimality error. Thus, when hidden domain identity is control-relevant, it must be encoded internally.
\subsection{From memory separation to a decodable local dynamics model}

The separation result shows that memory must retain decision-relevant domain information. It does not imply that the full transition kernel is decodable from memory. Decoding requires auxiliary probe goals whose values test one-step predictions, plus memory sufficient for those auxiliary values. A local history $h_t$ is reachable in domain $\sigma$ if it can occur with nonzero probability under some admissible policy in that domain. For a reachable history $h_t$, let $s_t(h_t)$ denote its final state. Let $\mathcal{F}\subseteq\{\phi:\State\to[-1,1]\}$ be a class of bounded test functions. For probability measures $\nu,\tilde{\nu}\in\mathcal{P}(\State)$, define the integral probability metric
\[
d_{\mathcal{F}}(\nu,\tilde{\nu})
:=
\sup_{\phi\in\mathcal{F}}
\left|
\int \phi\,d\nu
-
\int \phi\,d\tilde{\nu}
\right|.
\]
All local decoder errors below are measured in this metric. Let $\GoalsAux$ be an auxiliary family of probe goals. Each auxiliary goal $\bar{\psi}\in\GoalsAux$ has its own bounded utility $u_{\bar{\psi}}:\State\times\State\to\mathbb{R}$, with value functions defined as above after replacing $\psi$ by $\bar\psi$. These goals are used only to test local transition dynamics.

\begin{definition}[Goal-induced test class]
\label{def:testclass}
The class $\mathcal{F}$ is $\varepsilon_{\mathrm{span}}$-generated by $\GoalsAux$ if, for every $\phi\in\mathcal{F}$, there exist an auxiliary goal $\bar{\psi}_\phi\in\GoalsAux$ and a bounded proxy function $W_\phi:\State\to[-1,1]$ such that $\|W_\phi-\phi\|_\infty\le \varepsilon_{\mathrm{span}}$, and, for every domain $\sigma$, reachable state $s$, and action $a$,
\[
Q^\star_{\sigma,\bar{\psi}_\phi}(s,a)
=
\int W_\phi(s')\,P^\sigma(ds'\mid s,a).
\]
\end{definition}

Thus, auxiliary goals act as one-step predictive probes: their optimal action-values equal expectations of test functions under the true interventional kernel.

\begin{definition}[Value-sufficient memory]
\label{def:valuesufficient}
Memory is $\varepsilon_{\mathrm{val}}$-value-sufficient for $\GoalsAux$ if there exists a value head $q_M:\Mem\times\GoalsAux\times\Act\to\mathbb{R}$ such that, for every domain $\sigma$, every reachable history $h_t$ in that domain, every auxiliary goal $\bar{\psi}\in\GoalsAux$, and every action $a\in\Act$,
\[
\left|
q_M(f(h_t),\bar{\psi},a)
-
Q^\star_{\sigma,\bar{\psi}}(s_t(h_t),a)
\right|
\le
\varepsilon_{\mathrm{val}}.
\]
\end{definition}
Intuitively, $\varepsilon_{\mathrm{val}}$-value-sufficiency says that the memory state $M_t=f(h_t)$ retains all information needed to predict the optimal auxiliary probe values at the current local situation. The value head $q_M$ need not know the domain label $\sigma$ or the full history $h_t$ explicitly; any domain information relevant to the auxiliary values must already be encoded in $M_t$. Thus, if two reachable histories have the same memory state, then the auxiliary probes cannot assign very different optimal action-values to them, up to error $\varepsilon_{\mathrm{val}}$.
\begin{assumption}[Goal-induced local test class]
\label{ass:test_rich}
There exist a test class $\mathcal{F}\subseteq\{\phi:\State\to[-1,1]\}$ and an auxiliary goal family $\GoalsAux$ such that $\mathcal{F}$ is $\varepsilon_{\mathrm{span}}$-generated by $\GoalsAux$.
\end{assumption}

This assumption specifies what one-step distinctions the auxiliary goals can probe. The integral probability metric $d_{\mathcal{F}}$ is the local model-error metric used in Theorem~\ref{thm:decoder}; if $\mathcal{F}$ is too small, only the corresponding coarse aspects of the local transition kernel are identified.

\begin{assumption}[Planning stability under $d_{\mathcal{F}}$]
\label{ass:planning_stability}
There exists a nondecreasing function $\mathfrak{S}:\mathbb{R}_{\ge 0}\to\mathbb{R}_{\ge 0}$ with $\mathfrak{S}(0)=0$ such that the following holds. Let $\widehat P$ be any family of approximate local kernels indexed by reachable histories and actions. If, for some $\varepsilon\ge 0$,
\[
\sup_{\sigma\in\SigmaSet}\ 
\sup_{h_t\in\mathcal{H}^{\mathrm{reach}}_\sigma}\ 
\sup_{a\in\Act}
 d_{\mathcal{F}}\!\left(
\widehat P(\cdot\mid h_t,a),
P^\sigma(\cdot\mid s_t(h_t),a)
\right)
\le
\varepsilon,
\]
where $\mathcal{H}^{\mathrm{reach}}_\sigma$ is the set of histories reachable in domain $\sigma$, then, for every goal $\psi\in\GoalsN$, domain $\sigma$, reachable history $h_t\in\mathcal{H}^{\mathrm{reach}}_\sigma$, and action $a\in\Act$,
\[
\left|
\widehat Q^\star_{\psi,\widehat P}(h_t,a)
-
Q^\star_{\sigma,\psi}(s_t(h_t),a)
\right|
\le
\mathfrak{S}(\varepsilon).
\]
Here $\widehat Q^\star_{\psi,\widehat P}(h_t,a)$ is the optimal action-value obtained by Bellman planning from history $h_t$ using $\widehat P$ and utility $u_\psi$.
\end{assumption}

\begin{theorem}[Approximate decoder existence]
\label{thm:decoder}
Suppose Assumption~\ref{ass:test_rich} holds and memory is $\varepsilon_{\mathrm{val}}$-value-sufficient for $\GoalsAux$ in the sense of Definition~\ref{def:valuesufficient}. Then:

\begin{enumerate}[leftmargin=1.25em]
\item \textbf{Same memory implies similar local kernels.}
Let $h_t$ be a reachable history in domain $\sigma$, and let $\tilde h_{\tilde t}$ be a reachable history in domain $\tilde{\sigma}$. Define $m:=f(h_t)=f(\tilde h_{\tilde t})$, $s:=s_t(h_t)$, and $\tilde s:=s_{\tilde t}(\tilde h_{\tilde t})$. If the two histories have the same memory state $m$, then, for every action $a\in\Act$,
\[
d_{\mathcal{F}}\!\left(
P^\sigma(\cdot\mid s,a),
P^{\tilde{\sigma}}(\cdot\mid \tilde s,a)
\right)
\le
2(\varepsilon_{\mathrm{val}}+\varepsilon_{\mathrm{span}}).
\]

\item \textbf{A local transition decoder exists.}
There exists a decoder $D_P:\Mem\times\Act\to\mathcal{P}(\State)$ such that, for every domain $\sigma$, every reachable history $h_t$ in that domain, and every action $a\in\Act$,
\[
d_{\mathcal{F}}\!\left(
D_P(f(h_t),a),
P^\sigma(\cdot\mid s_t(h_t),a)
\right)
\le
2(\varepsilon_{\mathrm{val}}+\varepsilon_{\mathrm{span}}).
\]
For memory states that are never reached, $D_P$ may be defined arbitrarily.

\item \textbf{Decoded dynamics support approximate planning.}
If Assumption~\ref{ass:planning_stability} also holds, define $\widehat Q^\star_{\psi,D_P}(h_t,a)$ as the action-value assigned at reachable history $h_t$ by planning with decoded kernels $D_P(f(h_t),a)$ and utility $u_\psi$. Then, for every goal $\psi\in\GoalsN$, domain $\sigma$, reachable history $h_t$ in that domain, and action $a\in\Act$,
\[
\left|
\widehat Q^\star_{\psi,D_P}(h_t,a)
-
Q^\star_{\sigma,\psi}(s_t(h_t),a)
\right|
\le
\mathfrak{S}\!\left(
2(\varepsilon_{\mathrm{val}}+\varepsilon_{\mathrm{span}})
\right).
\]
\end{enumerate}
\end{theorem}

The decoder result is an identification-by-tests argument. Auxiliary goals recover expectations of proxy tests $W_\phi$ under the true interventional kernel. If memory is value-sufficient, equal memory states imply nearly equal expectations for all such proxies. Since these proxies uniformly approximate the tests in $\mathcal{F}$, the corresponding local kernels are close in $d_{\mathcal{F}}$. Thus memory supports an approximate decoder of $P^\sigma(\cdot\mid s,\doop{A=a})$ on reachable local situations, up to the distinctions measured by $\mathcal{F}$.
\newcommand{\PLACEHOLDER}[1]{\textbf{[\textcolor{red}{#1}]}}
\section{Experiments}
\label{sec:experiments}

We design experiments around a minimal controlled environment that satisfies Assumptions \ref{ass:shared_bottleneck} by construction, allowing us to test our theoretical claims directly. We ask three questions:
\begin{enumerate}[label=(\roman*),leftmargin=2em]
    \item Does uniform near-optimality \emph{require} memory, as Theorem~\ref{thm:separation} states?
    \item Is the required domain information \emph{represented separately} in memory at the bottleneck?
    \item Does a value-sufficient memory admit an approximate decoder of local interventional dynamics, as Theorem~\ref{thm:decoder} suggests?
\end{enumerate}

\subsection{The \textsc{ForkWorld} Environment}
\label{sec:fork_world}

\paragraph{Layout.}
\textsc{ForkWorld} is a T-junction gridworld (Figure~\ref{fig:fork_world}) where a corridor of length $c$ leads from the start cell $s_0$ to a fork (bottleneck) $s^\dagger$. From $s^\dagger$, an upper arm of length $\ell$ leads to $G_\text{up}$ and a lower arm leads to $G_\text{down}$. All movement is deterministic; the state is the agent's $(x, y)$ position.

\paragraph{Domain family.}
The domain latent $\sigma \in \Sigma = \{\textsc{normal}, \textsc{swapped}\}$ intervenes on the vertical action mapping:
\begin{equation}\notag
  \textsc{normal}: \; \texttt{up} \!\mapsto\! (0,{+}1),\; \texttt{down} \!\mapsto\! (0,{-}1)
  \qquad
  \textsc{swapped}: \; \texttt{up} \!\mapsto\! (0,{-}1),\; \texttt{down} \!\mapsto\! (0,{+}1).
\end{equation}
The domain is \emph{not} provided to the agent; it must be inferred from experienced transitions. Horizontal and stay actions are domain-invariant. 

\paragraph{Goal family and reward.}
The goal $\psi \in \Psi = \{\textsc{up},\textsc{down}\}$ specifies which endpoint to reach and is included in the observation. Horizontal and stay actions are domain-invariant. The reward function is goal-conditional and domain-invariant: reaching the correct endpoint yields $+1$; reaching the wrong endpoint or timing out yields $0$. Each transition is subject to a cost of $0.01$.

\paragraph{The four-way bottleneck.}
At $s^\dagger$, the optimal action depends jointly on the goal and the domain as illustrated in Figure~\ref{fig:fork_world}:
The optimal action sets are disjoint across domains for every goal ($G_\textsc{normal} \cap G_\textsc{swapped} = \emptyset$), satisfying Assumption~\ref{ass:shared_bottleneck}. The goal $\psi$ is observable, but the domain $\sigma$ is not. Any fork policy — including a uniformly random one — therefore achieves exactly $50\%$ success: for each goal, one of \{\texttt{up}, \texttt{down}\} is correct under $\sigma = \textsc{normal}$ and the other under $\sigma = \textsc{swapped}$, so guessing between them gives $50\%$ per goal and hence $50\%$ overall.
A memoryless agent cannot exceed this ceiling; one that fails to commit to either arm (e.g.\ by dithering at $s^\dagger$) can fall below it.


\subsection{Agents and Baselines}
\label{sec:agents}
We compare six agent configurations spanning the memory--oracle spectrum: \textbf{Memoryless DQN} \citep{mnih2015human}, a two-layer MLP that observes only the current state and goal $[x,y,\psi_{\mathrm{id}},1]$ and therefore cannot accumulate evidence about the hidden domain $\sigma$; \textbf{Stacked DQN}, the same MLP augmented with a bounded history window of width $k=3$ and one-hot previous actions; \textbf{DRQN (states)} \citep{hausknecht2015deep}, a GRU-based Deep Recurrent Q-Network that processes the full episode sequence using only state observations, with hidden state serving as implicit memory $M_t=f(h_t)$; \textbf{DRQN (states + actions)}, which additionally concatenates the previous action $a_{t-1}$ as a one-hot vector to expose transition information $(s_{t-1},a_{t-1},s_t)$ to the recurrent cell; \textbf{Oracle DQN}, a memoryless DQN given the domain label $\sigma_{\mathrm{id}}\in\{0,1\}$, establishing the
performance ceiling without memory when the domain is observed; and \textbf{Oracle DRQN}, which combines recurrence with the observed domain label and establishes the ceiling when both memory and $\sigma$ are available.

All agents use $\epsilon$-greedy exploration with linear decay ($\epsilon: 1.0 \to 0.05$ over $3{,}000$ episodes), Adam optimiser ($\text{lr}=10^{-3}$), $\gamma=0.99$, and are trained for $5{,}000$ episodes. DQN uses a transition replay buffer (capacity $20{,}000$); DRQN uses an episode replay buffer (capacity $2{,}000$ episodes) with segment sampling (active length $20$, burn-in $4$). All results are averaged over $8$ seeds; shaded regions show $\pm 1$ standard error.

\subsection{Memory is Necessary for Uniform Competence}
\label{sec:memory_necessary}
Figure~\ref{fig:learning_curves} shows success rates versus training steps for all six agents. Table~\ref{tab:final_success} reports the final evaluation success rate, broken down by (goal $\times$ domain) condition.

\begin{figure}[t]
  \centering
  \begin{subfigure}[t]{0.48\linewidth}
    \includegraphics[width=\linewidth]{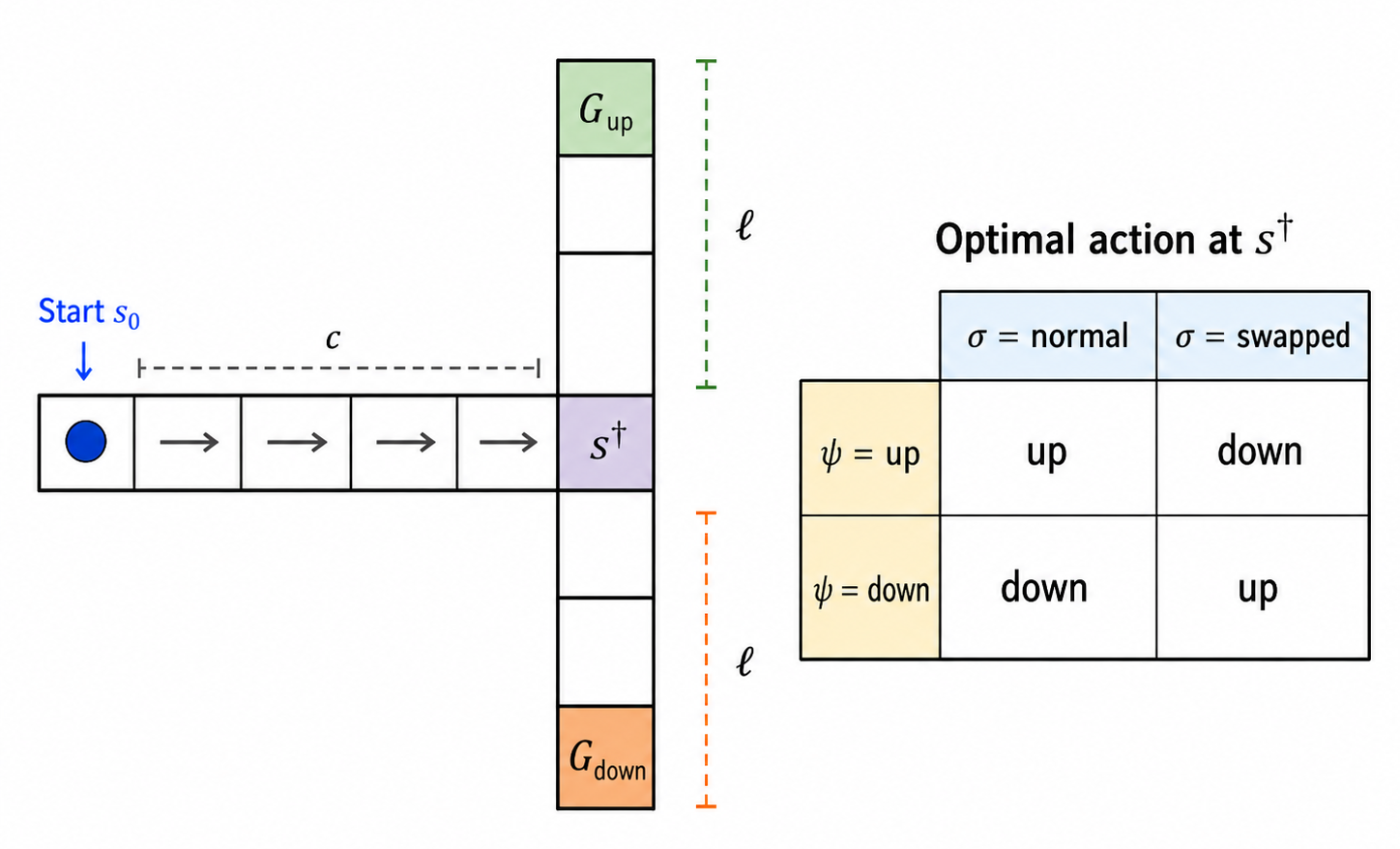}
    \caption{The \textsc{ForkWorld} environment.}
    \label{fig:fork_world}
  \end{subfigure}
  \hfill
  \begin{subfigure}[t]{0.48\linewidth}
    \includegraphics[width=\linewidth]{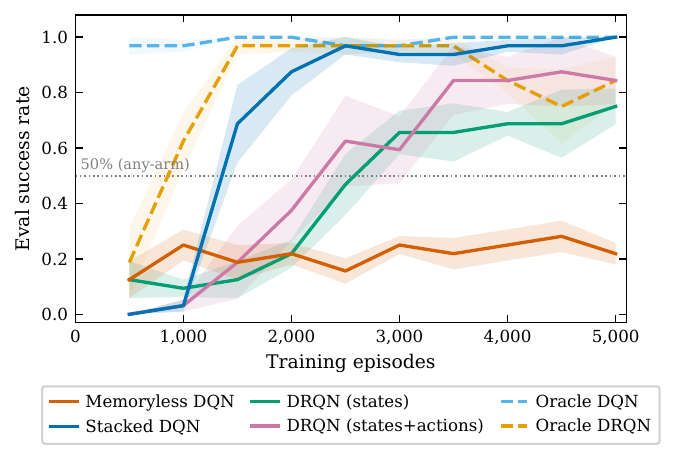}
    \caption{Eval. success rates (mean $\pm$ se) over training steps.}
    \label{fig:learning_curves}
  \end{subfigure}
  \caption{The environment and agent success rates over training.}
\end{figure}

\begin{table}[t]
  \centering
  \caption{\textbf{Final evaluation success rate} (mean $\pm$ 1\,SE across
    8 seeds), stratified by (goal $\times$ domain) condition.
    An agent solving all four conditions uniformly must infer the domain
    from memory.  N\,=\,normal, S\,=\,swapped, UP\,=\,up goal, DN\,=\,down goal.}
  \label{tab:final_success}
  \resizebox{\linewidth}{!}{%
  \begin{tabular}{lcccccc}
    \toprule
    & \multicolumn{4}{c}{Per-condition success rate (mean $\pm$ 1\,SE, 8 seeds)} & & \\
    \cmidrule(lr){2-5}
    Agent
      & $(\psi\!=\!\textsc{up},\,\sigma\!=\!\textsc{n})$
      & $(\psi\!=\!\textsc{up},\,\sigma\!=\!\textsc{s})$
      & $(\psi\!=\!\textsc{dn},\,\sigma\!=\!\textsc{n})$
      & $(\psi\!=\!\textsc{dn},\,\sigma\!=\!\textsc{s})$
      & Overall
      & $\Delta$ Oracle \\
    \midrule
    Memoryless DQN
      & $0.25 \pm 0.15$ & $0.12 \pm 0.12$
      & $0.25 \pm 0.15$ & $0.12 \pm 0.12$
      & $0.19 \pm 0.06$
      & $-0.81$ \\
    Stacked DQN ($k=3$, +actions)
      & $0.88 \pm 0.12$ & $0.88 \pm 0.12$
      & $0.88 \pm 0.12$ & $0.88 \pm 0.12$
      & $0.88 \pm 0.08$
      & $-0.12$ \\
    DRQN (states)
      & $0.75 \pm 0.15$ & $1.00 \pm 0.00$
      & $0.62 \pm 0.17$ & $0.62 \pm 0.17$
      & $0.75 \pm 0.09$
      & $-0.25$ \\
    DRQN (states + actions)
      & $0.88 \pm 0.12$ & $0.75 \pm 0.15$
      & $0.75 \pm 0.15$ & $0.75 \pm 0.15$
      & $0.78 \pm 0.12$
      & $-0.22$ \\
    \midrule
    Oracle DQN
      & $1.00 \pm 0.00$ & $1.00 \pm 0.00$
      & $1.00 \pm 0.00$ & $1.00 \pm 0.00$
      & $1.00 \pm 0.00$
      & --- \\
    Oracle DRQN
      & $0.75 \pm 0.15$ & $0.88 \pm 0.12$
      & $0.88 \pm 0.12$ & $0.62 \pm 0.17$
      & $0.78 \pm 0.12$
      & --- \\
    \bottomrule
  \end{tabular}}
  \vspace{-2ex}
\end{table}
The memoryless DQN plateaus near $0.19$ — well below the $0.50$ baseline achievable by any policy that commits to either arm. The failure is more severe than a simple domain-guessing deficit:
conflicting gradient signals at $s^\dagger$ (the same observation must support contradictory optimal actions across the two domains) prevent the Q-function from converging to a stable fork preference. The agent frequently dithers at $s^\dagger$ or retreats to the corridor, resulting in horizon-truncated episodes with no success reward. Theorem~\ref{thm:separation} predicts this: without domain-separated memory, the shared policy head cannot satisfy both domains' disjoint action sets simultaneously, and the resulting gradient conflict prevents even the trivial $50\%$ ceiling from being reached.

Memory-equipped agents learn to infer the domain from trajectory context and close 70--85\% of the gap to Oracle DQN ($1.00$) within $5{,}000$ episodes. Stacked DQN ($0.88$) matches or exceeds the DRQN variants ($0.75$--$0.78$) in overall success rate; the explicit $(s_{t-1}, a_{t-1}, s_t)$ window in the stacked representation is sufficient to distinguish the two arm
trajectories after a single probe. 

\subsection{Representation Separation at the Bottleneck}
\label{sec:rep_sep}
Theorem~\ref{thm:separation} predicts not just that memory agents succeed, but that their internal state must be \emph{TV-separated by domain} at $s^\dagger$. We test this directly with a linear probe.

After training, we run $200$ greedy rollout episodes per (goal $\times$ domain) condition ($800$ episodes total), recording the agent's internal representation $M_{t^\dagger}$ at every visit to the bottleneck. For DRQN this is the GRU hidden state; for DQN it is the last-hidden-layer activation. We fit a logistic regression\footnote{Balanced class weights to account for unequal visit counts across domains; evaluation by balanced-accuracy score (average per-class recall) with 5-fold CV.} to predict $\sigma \in \{\textsc{normal}, \textsc{swapped}\}$ from $M_{t^\dagger}$, and separately for first-time vs.\ subsequent visits to $s^\dagger$. Figure~\ref{fig:probe_visit} reports the results.

\emph{First visit.}
At the \emph{first} arrival at $s^\dagger$ — before the agent has entered either arm — domain classification accuracy is $0.48 \pm 0.00$ (SE) for \emph{all} four agents, indistinguishable from chance. This is the expected information-theoretic floor: the corridor provides no vertical-transition evidence, so $\sigma$ is unobservable.

\emph{Subsequent visits.}
After the agent enters an arm and returns, the picture changes sharply. For memory-equipped agents, domain accuracy on subsequent visits rises to $0.97 \pm 0.02$ (DRQN, states), $0.97 \pm 0.03$ (Stacked DQN, $k\!=\!3$), and $1.00 \pm 0.00$ (DRQN, states\,+\,actions). For the memoryless DQN it remains $0.50 \pm 0.00$ - exactly chance. The memory-equipped agents' representations therefore satisfy the TV-separation condition of Theorem~\ref{thm:separation} at $s^\dagger$ after the first arm probe, while the memory-less agent's representation is permanently domain-agnostic.

\begin{figure}[t]
  \centering
  \begin{minipage}[b]{0.45\linewidth}
    \centering
    \includegraphics[width=\linewidth]{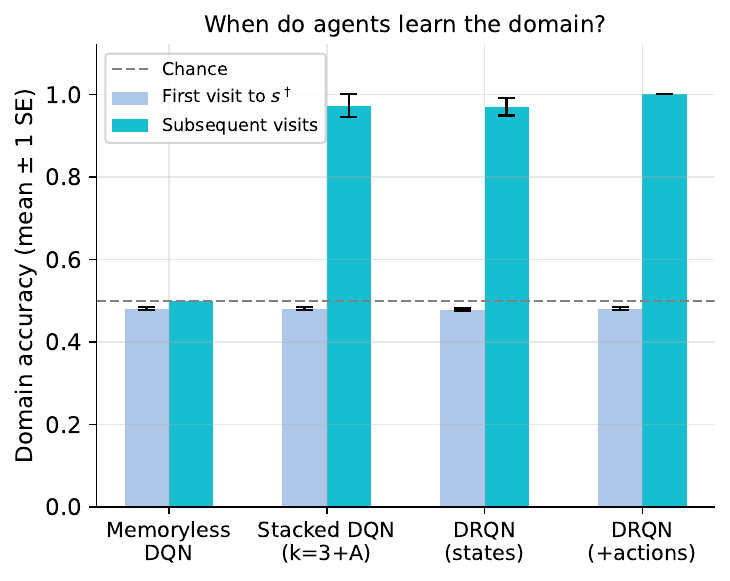}
    \captionof{figure}{\textbf{Domain probe accuracy by visit number.}
    On the first visit, all agents are at chance. On subsequent visits, DRQN
    representations become near-perfectly domain-separable, while the
    memoryless DQN remains at chance.}
    \label{fig:probe_visit}
  \end{minipage}
  \hfill 
  \begin{minipage}[b]{0.50\linewidth}
    \centering
  \centering
  \small
  \begin{tabular}{lcc}
    \toprule
    Agent & Decoder & Baseline \\
    \midrule
    DQN (memoryless)  & $0.28 \pm 0.08$ & $0.28$ \\
    Stacked DQN       & $0.33 \pm 0.01$ & $0.47$ \\
    DRQN (states)     & $0.17 \pm 0.03$ & $0.46$ \\
    DRQN (+actions)   & $0.28 \pm 0.04$ & $0.48$ \\
    Oracle DQN        & $0.00 \pm 0.00$ & $0.50$ \\
    Oracle DRQN       & $0.00 \pm 0.00$ & $0.52$ \\
    \bottomrule
  \end{tabular}
  \captionof{table}{\textbf{Transition decoder error at $s^\dagger$} (mean\,$\pm$\,1\,SE,
    8 seeds). Baseline: majority-class predictor ignoring $M_t$.
    Lower decoder error is better; gap vs.\ baseline measures improvement.}
  \label{tab:decoder}

  \end{minipage}
\end{figure}

\subsection{Memory Supports a Decodable World Model}
\label{sec:decoder}
Theorem~\ref{thm:decoder} predicts that value-sufficient memory admits an approximate decoder $D_P : \mathcal{M} \times \mathcal{A} \to \mathcal{P}(\mathcal{S})$ of the local interventional kernel.  We test this by fitting a transition decoder from the trained agent's representations.

We collect $(M_{t}, a_t, s_{t+1})$ tuples from 800 (200 episodes per (goal × domain) condition, 4 conditions) greedy rollout episodes.  A two-layer MLP decoder is trained to predict the next state $s_{t+1}$ from $(M_t, \text{onehot}(a_t))$, using a held-out $20\%$ split for evaluation. Prediction error is measured as the fraction of incorrectly predicted next states.

Table~\ref{tab:decoder} reports fork-state decoder error alongside the
majority-class baseline (predict the most frequent next state, ignoring $M_t$).
Because a more variable fork policy spreads transitions over more next-states,
baselines differ across agents; reduction relative to each agent's own baseline
is the correct comparison.
Oracle agents achieve $0.00$ error.
DRQN\,(states) reduces fork error from $0.46$ to $0.17$ ($63\%$), and
DRQN\,(+actions) from $0.48$ to $0.28$ ($41\%$), both consistent with
Theorem~\ref{thm:decoder}.
The memoryless DQN matches its baseline exactly --- zero reduction ---
confirming that $M_t$ carries no recoverable transition information at $s^\dagger$.
All-state errors are ${\leq}0.05$ across all agents (not shown), since
corridor transitions are domain-invariant.


\section{Discussion}\label{sec:discussion}

This paper shows that memory is necessary for generalist reinforcement learning when latent domains require different actions at the same observation. Unlike prior world-model necessity results, we do not assume the domain index is given; the agent must infer it from its own trajectory.
This is not a standard POMDP sufficiency claim. POMDP theory says a belief state can support optimal control; our result says any near-optimal learned memory must encode the hidden context when it is action-relevant. Thus, near-optimality itself forces memory separation.
Under suitable probe goals, this same memory also supports decoding of local interventional dynamics. Memory is therefore not just for choosing actions, but for carrying the information needed to reconstruct a local world model. The ForkWorld experiments illustrate these claims: memoryless agents fail under domain ambiguity, while memory-equipped agents separate latent domains and support transition decoding.

\newpage
\bibliographystyle{plainnat}
\bibliography{refs}
\appendix
\section*{Technical appendices and supplementary material}
\section{Proofs for Section~\ref{sec:method}}

\subsection{Proof of Theorem~\ref{thm:separation}}

\begin{lemma}[Action-set concentration at the bottleneck]
\label{lem:action_concentration}
Fix a pair $(\sigma,\tilde{\sigma})\in\mathcal{D}_{\mathrm{inc}}$, and let $\psi$, $s^\dagger$, $t_\sigma$, $t_{\tilde{\sigma}}$, $G_\sigma$, and $G_{\tilde{\sigma}}$ be the witnesses given by Assumption~\ref{ass:shared_bottleneck}. Under the uniform value-near-optimality hypothesis of Theorem~\ref{thm:separation}, for each $\xi\in\{\sigma,\tilde{\sigma}\}$ define
\[
\alpha_\xi(\cdot)
:=
\mathcal{L}^{\pi}_{\xi}(A_{t_\xi}\in\cdot\mid S_{t_\xi}=s^\dagger,S_0=s_0).
\]
Then
\[
\alpha_\xi(G_\xi)
\ge
1-
\frac{\varepsilon_{\mathrm{opt}}}{\rho\gamma^{t_\xi}\varepsilon_{\mathrm{gap}}}.
\]
\end{lemma}

\begin{proof}
Fix $\xi\in\{\sigma,\tilde{\sigma}\}$ and abbreviate $t:=t_\xi$ and $G:=G_\xi$. Since $\varepsilon_{\mathrm{opt}}\le\bar\varepsilon_{\mathrm{opt}}$, the theorem's value-near-optimality condition and Assumption~\ref{ass:shared_bottleneck}(1) imply
\[
\Pr_\xi(S_t=s^\dagger\mid \pi,s_0)\ge \rho.
\]

Fix $\nu>0$. Choose $a_\nu\in G$ whose value is within $\nu$ of the best value in $G$, so $Q^\star_{\xi,\psi}(s^\dagger,a_\nu)\ge \sup_{a\in G}Q^\star_{\xi,\psi}(s^\dagger,a)-\nu$. Construct a comparison policy $\hat\pi_\nu$ that agrees with $\pi$ up to time $t-1$. At time $t$, if $S_t=s^\dagger$, it takes $a_\nu$ and then follows a continuation policy within $\nu$ of optimal after taking $a_\nu$; outside $\{S_t=s^\dagger\}$, it behaves as $\pi$.

Conditioned on $S_t=s^\dagger$, the continuation value of $\pi$ from time $t$ onward is at most
\[
\int_{\Act} Q^\star_{\xi,\psi}(s^\dagger,a)\,\alpha_\xi(da),
\]
while $\hat\pi_\nu$ obtains at least $Q^\star_{\xi,\psi}(s^\dagger,a_\nu)-\nu$. For $a\in G$, the choice of $a_\nu$ gives $Q^\star_{\xi,\psi}(s^\dagger,a_\nu)\ge Q^\star_{\xi,\psi}(s^\dagger,a)-\nu$. For $b\notin G$, Assumption~\ref{ass:shared_bottleneck}(2) gives $Q^\star_{\xi,\psi}(s^\dagger,a_\nu)\ge Q^\star_{\xi,\psi}(s^\dagger,b)+\varepsilon_{\mathrm{gap}}-\nu$. Hence
\[
Q^\star_{\xi,\psi}(s^\dagger,a_\nu)
-
\int_{\Act}Q^\star_{\xi,\psi}(s^\dagger,a)\,\alpha_\xi(da)
\ge
\varepsilon_{\mathrm{gap}}\big(1-\alpha_\xi(G)\big)-\nu.
\]
Therefore,
\[
V^{\hat\pi_\nu,\psi}_{\xi}(s_0)-V^{\pi,\psi}_{\xi}(s_0)
\ge
\Pr_\xi(S_t=s^\dagger\mid \pi,s_0)\gamma^t
\Big(
\varepsilon_{\mathrm{gap}}\big(1-\alpha_\xi(G)\big)-2\nu
\Big).
\]
Since $V^{\star,\psi}_{\xi}(s_0)\ge V^{\hat\pi_\nu,\psi}_{\xi}(s_0)$ and $V^{\star,\psi}_{\xi}(s_0)-V^{\pi,\psi}_{\xi}(s_0)\le\varepsilon_{\mathrm{opt}}$,
\[
\varepsilon_{\mathrm{opt}}
\ge
\rho\gamma^t
\Big(
\varepsilon_{\mathrm{gap}}\big(1-\alpha_\xi(G)\big)-2\nu
\Big).
\]
Letting $\nu\downarrow 0$ yields the claim.
\end{proof}

\begin{proof}[Proof of Theorem~\ref{thm:separation}]
Fix $(\sigma,\tilde{\sigma})\in\mathcal{D}_{\mathrm{inc}}$, and let $\psi$, $s^\dagger$, $t_\sigma$, $t_{\tilde{\sigma}}$, $G_\sigma$, and $G_{\tilde{\sigma}}$ be the witnesses from Assumption~\ref{ass:shared_bottleneck}. Define
\[
\Lambda_\sigma
:=
\Pr_\sigma(M_{t_\sigma}\in\cdot\mid S_{t_\sigma}=s^\dagger,\pi,s_0),
\qquad
\Lambda_{\tilde{\sigma}}
:=
\Pr_{\tilde{\sigma}}(M_{t_{\tilde{\sigma}}}\in\cdot\mid S_{t_{\tilde{\sigma}}}=s^\dagger,\pi,s_0),
\]
and
\[
\eta_\sigma
:=
\frac{\varepsilon_{\mathrm{opt}}}{\rho\gamma^{t_\sigma}\varepsilon_{\mathrm{gap}}},
\qquad
\eta_{\tilde{\sigma}}
:=
\frac{\varepsilon_{\mathrm{opt}}}{\rho\gamma^{t_{\tilde{\sigma}}}\varepsilon_{\mathrm{gap}}}.
\]

Let $F(m):=\pi_M(G_\sigma\mid m,\psi)$, so $0\le F\le 1$. By the policy-head definition and Lemma~\ref{lem:action_concentration},
\[
\int F\,d\Lambda_\sigma
=
\Pr_\sigma(A_{t_\sigma}\in G_\sigma\mid S_{t_\sigma}=s^\dagger,\pi,s_0)
\ge
1-\eta_\sigma.
\]
Because $G_\sigma\cap G_{\tilde{\sigma}}=\emptyset$,
\[
\pi_M(G_\sigma\mid m,\psi)
\le
1-\pi_M(G_{\tilde{\sigma}}\mid m,\psi)
\qquad
\text{for every }m.
\]
Integrating with respect to $\Lambda_{\tilde{\sigma}}$ and applying Lemma~\ref{lem:action_concentration} again gives
\[
\int F\,d\Lambda_{\tilde{\sigma}}
\le
1-
\Pr_{\tilde{\sigma}}(A_{t_{\tilde{\sigma}}}\in G_{\tilde{\sigma}}\mid S_{t_{\tilde{\sigma}}}=s^\dagger,\pi,s_0)
\le
\eta_{\tilde{\sigma}}.
\]
Therefore,
\[
\int F\,d\Lambda_\sigma-
\int F\,d\Lambda_{\tilde{\sigma}}
\ge
1-\eta_\sigma-\eta_{\tilde{\sigma}}.
\]
Since $0\le F\le 1$, the variational characterization of total variation implies
\[
\mathrm{TV}(\Lambda_\sigma,\Lambda_{\tilde{\sigma}})
\ge
\left[
1-\eta_\sigma-\eta_{\tilde{\sigma}}
\right]_+.
\]
Finally, because $t_\sigma,t_{\tilde{\sigma}}\le T$,
\[
\eta_\sigma,\eta_{\tilde{\sigma}}
\le
\frac{\varepsilon_{\mathrm{opt}}}{\rho\gamma^T\varepsilon_{\mathrm{gap}}}.
\]
If $\varepsilon_{\mathrm{opt}}<\varepsilon_0:=\tfrac12\rho\gamma^T\varepsilon_{\mathrm{gap}}$, then $\eta_\sigma+\eta_{\tilde{\sigma}}<1$, so $\mathrm{TV}(\Lambda_\sigma,\Lambda_{\tilde{\sigma}})>0$. Hence the two conditional memory distributions differ.
\end{proof}

\subsection{Proof of Theorem~\ref{thm:decoder}}

\begin{proof}[Proof of Theorem~\ref{thm:decoder}]
We prove the three claims in order.

\paragraph{Part (1).}
Let $h_t$ be reachable in domain $\sigma$, and let $\tilde h_{\tilde t}$ be reachable in domain $\tilde{\sigma}$. Suppose $f(h_t)=f(\tilde h_{\tilde t})=:m$. Let $s:=s_t(h_t)$ and $\tilde s:=s_{\tilde t}(\tilde h_{\tilde t})$. Fix $a\in\Act$ and write
\[
K:=P^\sigma(\cdot\mid s,a),
\qquad
\tilde K:=P^{\tilde{\sigma}}(\cdot\mid \tilde s,a).
\]

Let $\phi\in\mathcal F$. Since $\mathcal F$ is $\varepsilon_{\mathrm{span}}$-generated by $\GoalsAux$, there exist $\bar\psi_\phi\in\GoalsAux$ and $W_\phi:\State\to[-1,1]$ with $\|W_\phi-\phi\|_\infty\le\varepsilon_{\mathrm{span}}$ and
\[
Q^\star_{\sigma,\bar\psi_\phi}(s,a)=\int W_\phi(s')\,K(ds'),
\qquad
Q^\star_{\tilde{\sigma},\bar\psi_\phi}(\tilde s,a)=\int W_\phi(s')\,\tilde K(ds').
\]
Value-sufficiency gives
\[
\left|q_M(m,\bar\psi_\phi,a)-Q^\star_{\sigma,\bar\psi_\phi}(s,a)\right|
\le
\varepsilon_{\mathrm{val}},
\qquad
\left|q_M(m,\bar\psi_\phi,a)-Q^\star_{\tilde{\sigma},\bar\psi_\phi}(\tilde s,a)\right|
\le
\varepsilon_{\mathrm{val}}.
\]
Thus,
\[
\left|
\int W_\phi\,dK-
\int W_\phi\,d\tilde K
\right|
\le
2\varepsilon_{\mathrm{val}}.
\]
Comparing $\phi$ and $W_\phi$,
\[
\left|
\int \phi\,dK-
\int \phi\,d\tilde K
\right|
\le
\left|\int(\phi-W_\phi)\,dK\right|
+
\left|\int W_\phi\,dK-\int W_\phi\,d\tilde K\right|
+
\left|\int(W_\phi-\phi)\,d\tilde K\right|
\le
2(\varepsilon_{\mathrm{val}}+\varepsilon_{\mathrm{span}}).
\]
Taking the supremum over $\phi\in\mathcal F$ yields
\[
d_{\mathcal F}(K,\tilde K)
\le
2(\varepsilon_{\mathrm{val}}+\varepsilon_{\mathrm{span}}).
\]

\paragraph{Part (2).}
Let
\[
\Mem_{\mathrm{reach}}
:=
\{f(h_t): h_t \text{ is reachable in some domain}\}.
\]
For each $m\in\Mem_{\mathrm{reach}}$, choose one representative reachable history $h^m$ with $f(h^m)=m$. Let $\sigma^m$ be its domain and $s^m$ its final state. Define
\[
D_P(m,a):=P^{\sigma^m}(\cdot\mid s^m,a),
\qquad
m\in\Mem_{\mathrm{reach}},\ a\in\Act.
\]
For unreachable memory states, define $D_P(m,a)$ arbitrarily.

Take any reachable history $h_t$ in domain $\sigma$, and let $s:=s_t(h_t)$ and $m:=f(h_t)$. The history $h_t$ and representative $h^m$ share the same memory state, so part (1) gives
\[
d_{\mathcal F}\big(D_P(m,a),P^\sigma(\cdot\mid s,a)\big)
\le
2(\varepsilon_{\mathrm{val}}+\varepsilon_{\mathrm{span}})
\qquad
\forall a\in\Act.
\]
Substituting $m=f(h_t)$ proves the decoder bound.

\paragraph{Part (3).}
Assume Assumption~\ref{ass:planning_stability}. Let $\varepsilon_D:=2(\varepsilon_{\mathrm{val}}+\varepsilon_{\mathrm{span}})$. By part (2), the decoded local kernel $D_P(f(h_t),a)$ is within $\varepsilon_D$ of $P^\sigma(\cdot\mid s_t(h_t),a)$ in $d_{\mathcal F}$ for every reachable $h_t$ and action $a$.

Define
\[
\widehat P_D(\cdot\mid h_t,a):=D_P(f(h_t),a).
\]
The bound from part (2) is exactly the premise of Assumption~\ref{ass:planning_stability} with $\widehat P=\widehat P_D$ and $\varepsilon=\varepsilon_D$. Therefore,
\[
\left|
\widehat Q^\star_{\psi,D_P}(h_t,a)
-
Q^\star_{\sigma,\psi}(s_t(h_t),a)
\right|
\le
\mathfrak S(\varepsilon_D)
\]
for every goal $\psi\in\GoalsN$, domain $\sigma$, reachable history $h_t$, and action $a$. Substituting $\varepsilon_D=2(\varepsilon_{\mathrm{val}}+\varepsilon_{\mathrm{span}})$ proves part (3), and hence the theorem.
\end{proof}
\newpage
\section{Implementation Details}
\label{app:implementation}

\subsection{Environment}
\label{app:env}

\textsc{ForkWorld} is a deterministic gridworld with corridor length $c=4$ and arm length $\ell=3$, giving a $4\times7$ grid of reachable cells.  The fork state $s^\dagger$ is at position $(3,3)$; the start state $s_0$ is at $(0,3)$; $G_\text{up}=(3,6)$ and $G_\text{down}=(3,0)$.  The action space has five discrete actions \{\texttt{up}, \texttt{down}, \texttt{left}, \texttt{right}, \texttt{stay}\}.  The observation is the vector $[x,\, y,\, \psi_\text{id},\, 1]^\top\in\mathbb{R}^4$, where $(x,y)$ is the agent's position, $\psi_\text{id}\in\{0,1\}$ encodes the goal, and the trailing $1$ serves as a bias input.  Oracle agents additionally receive $\sigma_\text{id}\in\{0,1\}$, giving a 5-dimensional observation $[x,\, y,\, \psi_\text{id},\, \sigma_\text{id},\, 1]^\top$.  Each episode samples $(\psi, \sigma)$ uniformly from the four combinations; evaluation is stratified over all four conditions (5 episodes each, 20 total) to eliminate sampling variance.  The episode horizon is $H=50$; reaching the correct goal yields reward $+1$, reaching the wrong goal or timing out yields $0$, and every transition incurs a step cost of $0.01$.

\subsection{Agent Architectures}
\label{app:architectures}

All Q-networks use a two-layer MLP head with hidden dimensions $[64, 64]$ and ReLU activations, outputting one logit per action.

\paragraph{Memoryless DQN.}  Processes only the current observation; network input is $\mathbb{R}^4$.

\paragraph{Stacked DQN ($k=3$, +actions).}  Maintains a sliding window of the last $k=3$ observations and $k-1=2$ interleaved one-hot previous actions.  The input dimension is $4k + 5(k-1) = 12 + 10 = 22$.  Both the current observation and the prior action are re-encoded at every step; the window is zero-padded at episode start.

\paragraph{DRQN (states).}  A single-layer GRU with hidden dimension $d=64$, consuming the raw 4-dimensional observation at each step.  The GRU hidden state serves as the implicit memory $M_t$.  The output head is a linear map $\mathbb{R}^{64}\to\mathbb{R}^5$.

\paragraph{DRQN (states + actions).}  Identical to DRQN (states) but concatenates the one-hot previous action $a_{t-1}\in\{0,1\}^5$ to the observation before the GRU, giving a 9-dimensional per-step input.

\paragraph{Oracle DQN.}  Memoryless MLP with a 5-dimensional input that includes $\sigma_\text{id}$.

\paragraph{Oracle DRQN.}  DRQN (states + actions) with a 10-dimensional per-step input (5-dimensional oracle observation concatenated with the one-hot previous action).

\subsection{Training Hyperparameters}
\label{app:hyperparams}

All agents share the same optimizer (AdamW, $\text{lr}=10^{-3}$), discount factor $\gamma=0.99$, $\varepsilon$-greedy schedule (linear from $\varepsilon=1.0$ to $\varepsilon=0.05$ over $3{,}000$ episodes, constant thereafter), and a total budget of $5{,}000$ episodes.  Evaluation uses $\varepsilon=0$ (greedy).  All results are averaged over 8 independent seeds.

\begin{table}[h]
  \centering
  \caption{Hyperparameters for DQN and DRQN variants.}
  \label{tab:hyperparams}
  \begin{tabular}{lcc}
    \toprule
    Hyperparameter & DQN variants & DRQN variants \\
    \midrule
    Hidden dims & $[64, 64]$ & --- \\
    GRU hidden dim & --- & $64$ \\
    Replay buffer & $20{,}000$ transitions & $2{,}000$ episodes \\
    Batch size & $64$ transitions & $32$ episodes \\
    Sequence length (active) & --- & $20$ steps \\
    Burn-in length & --- & $4$ steps \\
    Target network update & every $100$ updates & every $100$ updates \\
    Warmup (before training) & $200$ transitions & $50$ episodes \\
    Updates per env step & $1$ & $1$ \\
    \bottomrule
  \end{tabular}
\end{table}

For DRQN, each training sample is a segment of total length $\text{burn-in} + \text{seq\_len} + 1 = 25$ timesteps drawn from a uniformly sampled episode in the replay buffer; gradients are computed only over the active 20-step window.  Episodes shorter than 25 steps are zero-padded.

\subsection{Domain Probe}
\label{app:probe}

After training, we collect $200$ greedy rollout episodes per (goal $\times$ domain) condition ($800$ episodes total) and record the agent's internal representation $M_{t^\dagger}$ at every visit to $s^\dagger$.  For DRQN agents, $M_{t^\dagger}$ is the GRU hidden state $h_{t^\dagger}\in\mathbb{R}^{64}$; for DQN variants it is the penultimate-layer activation (after the second ReLU, before the output head), also in $\mathbb{R}^{64}$.

A logistic regression classifier (scikit-learn \texttt{LogisticRegression}; $C=1.0$; solver \texttt{lbfgs}; \texttt{max\_iter}=$1{,}000$; \texttt{class\_weight}=\texttt{balanced}) is fitted on standardised representations to predict $\sigma\in\{\textsc{normal},\textsc{swapped}\}$.  Performance is reported as balanced accuracy (mean per-class recall) under 5-fold stratified cross-validation.  Probes are run separately for first-time vs.\ subsequent visits to $s^\dagger$.

\subsection{Transition Decoder}
\label{app:decoder}

Using the same 800 greedy rollout episodes, we collect $(M_t, a_t, s_{t+1})$ tuples and train a two-layer MLP decoder to predict the next grid cell $s_{t+1}$ from $(M_t,\,\text{onehot}(a_t))$.  The decoder architecture is:
\[
  [\mathbb{R}^{64} \oplus \mathbb{R}^5] \xrightarrow{\text{Linear}} \mathbb{R}^{128} \xrightarrow{\text{ReLU}} \mathbb{R}^{128} \xrightarrow{\text{ReLU}} \mathbb{R}^{128} \xrightarrow{\text{Linear}} \mathbb{R}^{C},
\]
where $C=10$ is the number of reachable grid cells (4 corridor + 3 upper arm + 3 lower arm).  Training uses Adam ($\text{lr}=10^{-3}$), cross-entropy loss, batch size $256$, and $50$ epochs on an $80/20$ train/validation split.  Error is reported as the fraction of incorrectly predicted next cells on the held-out validation set.

\section{Compute}
Experimental environment is minimal in nature and all tasks were carried out on CPUs and one small GPU (RTX2080). Total Training/Post-Training time for each agent on GPU was 15-20 minutes.

\newpage
\section*{NeurIPS Paper Checklist}

The checklist is designed to encourage best practices for responsible machine learning research, addressing issues of reproducibility, transparency, research ethics, and societal impact. Do not remove the checklist: {\bf The papers not including the checklist will be desk rejected.} The checklist should follow the references and follow the (optional) supplemental material.  The checklist does NOT count towards the page
limit. 

Please read the checklist guidelines carefully for information on how to answer these questions. For each question in the checklist:
\begin{itemize}
    \item You should answer \answerYes{}, \answerNo{}, or \answerNA{}.
    \item \answerNA{} means either that the question is Not Applicable for that particular paper or the relevant information is Not Available.
    \item Please provide a short (1--2 sentence) justification right after your answer (even for \answerNA). 
\end{itemize}

{\bf The checklist answers are an integral part of your paper submission.} They are visible to the reviewers, area chairs, senior area chairs, and ethics reviewers. You will also be asked to include it (after eventual revisions) with the final version of your paper, and its final version will be published with the paper.

The reviewers of your paper will be asked to use the checklist as one of the factors in their evaluation. While \answerYes{} is generally preferable to \answerNo{}, it is perfectly acceptable to answer \answerNo{} provided a proper justification is given (e.g., error bars are not reported because it would be too computationally expensive'' or ``we were unable to find the license for the dataset we used''). In general, answering \answerNo{} or \answerNA{} is not grounds for rejection. While the questions are phrased in a binary way, we acknowledge that the true answer is often more nuanced, so please just use your best judgment and write a justification to elaborate. All supporting evidence can appear either in the main paper or the supplemental material, provided in appendix. If you answer \answerYes{} to a question, in the justification please point to the section(s) where related material for the question can be found.

IMPORTANT, please:
\begin{itemize}
    \item {\bf Delete this instruction block, but keep the section heading ``NeurIPS Paper Checklist"},
    \item  {\bf Keep the checklist subsection headings, questions/answers and guidelines below.}
    \item {\bf Do not modify the questions and only use the provided macros for your answers}.
\end{itemize}


\begin{enumerate}

\item {\bf Claims}
    \item[] Question: Do the main claims made in the abstract and introduction accurately reflect the paper's contributions and scope?
    \item[] Answer: \answerYes{} 
    \item[] Justification: Abstract and introduction accurately match the paper’s theoretical contributions (Theorems 1-2) and supporting experiments.

\item {\bf Limitations}
    \item[] Question: Does the paper discuss the limitations of the work performed by the authors?
    \item[] Answer: \answerYes{} 
    \item[] Justification: Limitations discussed in Section~\ref{sec:discussion}.

\item {\bf Theory assumptions and proofs}
    \item[] Question: For each theoretical result, does the paper provide the full set of assumptions and a complete (and correct) proof?
    \item[] Answer: \answerYes{} 
    \item[] Justification: All theoretical assumptions, definitions, theorems, and full proofs are stated in Section~\ref{sec:method} and Appendix.

    \item {\bf Experimental result reproducibility}
    \item[] Question: Does the paper fully disclose all the information needed to reproduce the main experimental results of the paper to the extent that it affects the main claims and/or conclusions of the paper (regardless of whether the code and data are provided or not)?
    \item[] Answer: \answerYes{} 
    \item[] Justification: The environment, architectures, training setup, and evaluation protocol are described sufficiently to reproduce the main experiments.

\item {\bf Open access to data and code}
    \item[] Question: Does the paper provide open access to the data and code, with sufficient instructions to faithfully reproduce the main experimental results, as described in supplemental material?
    \item[] Answer: \answerYes{} 
    \item[] Justification: Code submitted as part of supplementary material. 

\item {\bf Experimental setting/details}
    \item[] Question: Does the paper specify all the training and test details (e.g., data splits, hyperparameters, how they were chosen, type of optimizer) necessary to understand the results?
    \item[] Answer: \answerYes{} 
    \item[] Justification: The paper specifies architectures, replay buffers, optimizers, learning rates, exploration schedules, seeds, and evaluation setup.

\item {\bf Experiment statistical significance}
    \item[] Question: Does the paper report error bars suitably and correctly defined or other appropriate information about the statistical significance of the experiments?
    \item[] Answer: \answerYes{} 
    \item[] Justification: Results are reported over 8 seeds with mean $\pm$ standard error and probe evaluation details are provided.

\item {\bf Experiments compute resources}
    \item[] Question: For each experiment, does the paper provide sufficient information on the computer resources (type of compute workers, memory, time of execution) needed to reproduce the experiments?
    \item[] Answer: \answerYes{} 
    \item[] Justification: Compute details reported in Appendix Compute Section. 
    
\item {\bf Code of ethics}
    \item[] Question: Does the research conducted in the paper conform, in every respect, with the NeurIPS Code of Ethics \url{https://neurips.cc/public/EthicsGuidelines}?
    \item[] Answer: \answerYes{} 
    \item[] Justification: The research is consistent with the NeurIPS Code of Ethics and involves no unethical data collection or deployment.

\item {\bf Broader impacts}
    \item[] Question: Does the paper discuss both potential positive societal impacts and negative societal impacts of the work performed?
    \item[] Answer: \answerNA{} 
    \item[] Justification: Not applicable; our work is a theoretical RL paper.
    
\item {\bf Safeguards}
    \item[] Question: Does the paper describe safeguards that have been put in place for responsible release of data or models that have a high risk for misuse (e.g., pre-trained language models, image generators, or scraped datasets)?
    \item[] Answer: \answerNA{} 
    \item[] Justification: The work does not release high-risk generative models, scraped datasets, or dual-use systems requiring safeguards.

\item {\bf Licenses for existing assets}
    \item[] Question: Are the creators or original owners of assets (e.g., code, data, models), used in the paper, properly credited and are the license and terms of use explicitly mentioned and properly respected?
    \item[] Answer: \answerYes{} 
    \item[] Justification: Existing software libraries and prior assets are cited.

\item {\bf New assets}
    \item[] Question: Are new assets introduced in the paper well documented and is the documentation provided alongside the assets?
    \item[] Answer: \answerNA{} 
    \item[] Justification: The paper does not currently release new datasets, models, or benchmark assets.

\item {\bf Crowdsourcing and research with human subjects}
    \item[] Question: For crowdsourcing experiments and research with human subjects, does the paper include the full text of instructions given to participants and screenshots, if applicable, as well as details about compensation (if any)? 
    \item[] Answer: \answerNA{} 
    \item[] Justification: The work does not involve crowdsourcing or experiments with human participants.

\item {\bf Institutional review board (IRB) approvals or equivalent for research with human subjects}
    \item[] Question: Does the paper describe potential risks incurred by study participants, whether such risks were disclosed to the subjects, and whether Institutional Review Board (IRB) approvals (or an equivalent approval/review based on the requirements of your country or institution) were obtained?
    \item[] Answer: \answerNA{} 
    \item[] Justification: The work does not involve human subjects research requiring IRB review.

\item {\bf Declaration of LLM usage}
    \item[] Question: Does the paper describe the usage of LLMs if it is an important, original, or non-standard component of the core methods in this research? Note that if the LLM is used only for writing, editing, or formatting purposes and does \emph{not} impact the core methodology, scientific rigor, or originality of the research, declaration is not required.
    \item[] Answer: \answerNA{} 
    \item[] Justification: LLMs are not part of the paper’s core methodology or experimental setup.

\end{enumerate}
\end{document}